\documentclass[letterpaper]{article} 
\usepackage{aaai24}  
\usepackage{times}  
\usepackage{helvet}  
\usepackage{courier}  
\usepackage[hyphens]{url}  
\usepackage{graphicx} 

\usepackage{natbib}  
\usepackage{caption} 
\usepackage{algorithm}
\usepackage{algorithmic}
\usepackage{echodefs}
\usepackage{booktabs} 

\usepackage{xcolor}
\usepackage{newfloat}
\usepackage{listings}
\DeclareCaptionStyle{ruled}{labelfont=normalfont,labelsep=colon,strut=off} 
\floatstyle{ruled}
\newfloat{listing}{tb}{lst}{}
\floatname{listing}{Listing}
\usepackage{amsmath,amsfonts,amssymb}
\usepackage{amsthm}
\newtheorem{theorem}{Theorem}[section]
\newtheorem{lemma}[theorem]{Lemma}

\title{A Graph Dynamics Prior for Relational Inference}
\author{
    Liming Pan\textsuperscript{\rm 1,2}\equalcontrib,
    Cheng Shi\textsuperscript{\rm 3}\equalcontrib,
    Ivan Dokmani\'c\textsuperscript{\rm 3}\thanks{To whom correspondence should be addressed.}
}
\affiliations{
    \textsuperscript{\rm 1}School of Cyber Science and Technology,
  University of Science and Technology of China,
  Hefei, China 230026, \\
  \textsuperscript{\rm 2}School of Computer and Electronic Information, Nanjing Normal University, China 210023, \\
\textsuperscript{\rm 3}  Departement Mathematik und Informatik,
 Universit\"at Basel,
  Basel, Switzerland 4051\\
panlm99@gmail.com,    
firstname.lastname@unibas.ch
}

\usepackage{bibentry}

\begin{document}
\maketitle

\begin{abstract}
Relational inference aims to identify interactions between parts of a dynamical system from the observed dynamics. Current state-of-the-art methods fit the dynamics with a graph neural network (GNN) on a learnable graph. They use one-step message-passing GNNs---intuitively the right choice since non-locality of multi-step or spectral GNNs may confuse direct and indirect interactions. But the \textit{effective} interaction graph depends on the sampling rate and it is rarely localized to direct neighbors, leading to poor local optima for the one-step model. In this work, we propose a \textit{graph dynamics prior} (GDP) for relational inference. GDP constructively uses error amplification in non-local polynomial filters to steer the solution to the ground-truth graph. To deal with non-uniqueness, GDP simultaneously fits a ``shallow'' one-step model and a polynomial multi-step model with shared graph topology. Experiments show that GDP reconstructs graphs far more accurately than earlier methods, with remarkable robustness to under-sampling. Since appropriate sampling rates for unknown dynamical systems are not known a priori, this robustness makes GDP suitable for real applications in scientific machine learning. Reproducible code is available at https://github.com/DaDaCheng/GDP.
\end{abstract}

\section{Introduction}

Understanding interactions is key to understanding the function of dynamical systems in physics~\cite{arenas2008synchronization}, biology,  neuroscience~\cite{izhikevich2007dynamical}, epidemiology~\cite{pastor2015epidemic}, and sociology~\cite{castellano2009statistical}, to name a few. It is often time-consuming or even impossible to determine this structure experimentally: for example, neuronal connectivity is determined by painstaking analyses of electron microscopy images. On the other hand, there has been an explosion of availability of signal measurements. It is thus attractive to devise methods which determine interactions from the observed dynamics alone.

The seminal work on neural relational inference (NRI) showed that in some cases graph neural networks (GNNs) can perform well on this challenge~\cite{kipf2018neural}. GNN-based methods use a graph generator and a dynamics learner\footnote{Even though they are sometimes called differently, these two components are always present.}: the graph generator produces a candidate graph while the dynamics learner tries to match the dynamics to data, acting as a surrogate for the original system. Since the dynamics of a node only depends on its neighbors, it is intuitive to try and emulate it with a single-step message-passing GNN. Multi-step message passing or spectral GNNs may confuse direct and indirect neighbors. Indeed, single-step architectures appear in the original NRI work and its various adaptations~\cite{alet2019neural,lowe2022amortized,graber2020dynamic,
ha2023learning,zhu2022neural,
zhang2022universal,wang2022iterative}.

An implicit assumption in a single-step scheme is that the sampling rate (the number of samples per unit time) is sufficiently high. With a low sampling rate the effective interaction graph is non-local which causes a single-step surrogate to confuse direct and indirect interactions. Single-step message passing also limits the expressivity of neural surrogates. The vanilla GCN~\cite{kipf2017semi} and many other GNNs implicitly assume homophily and act as low-pass graph filters~\cite{zhu2020beyond}. Although they can handle certain heterophilic data~\cite{ma2022homophily}, single-layer GNNs can only implement a limited range of graph filters. Unknown nonlinear dynamics call for graph filters adaptive to data such as in ChebyNet~\cite{defferrard2016convolutional} or GPR-GNN~\cite{chien2021adaptive}.

\begin{figure*}[t]
    \centering
    \includegraphics[width=0.95\linewidth]{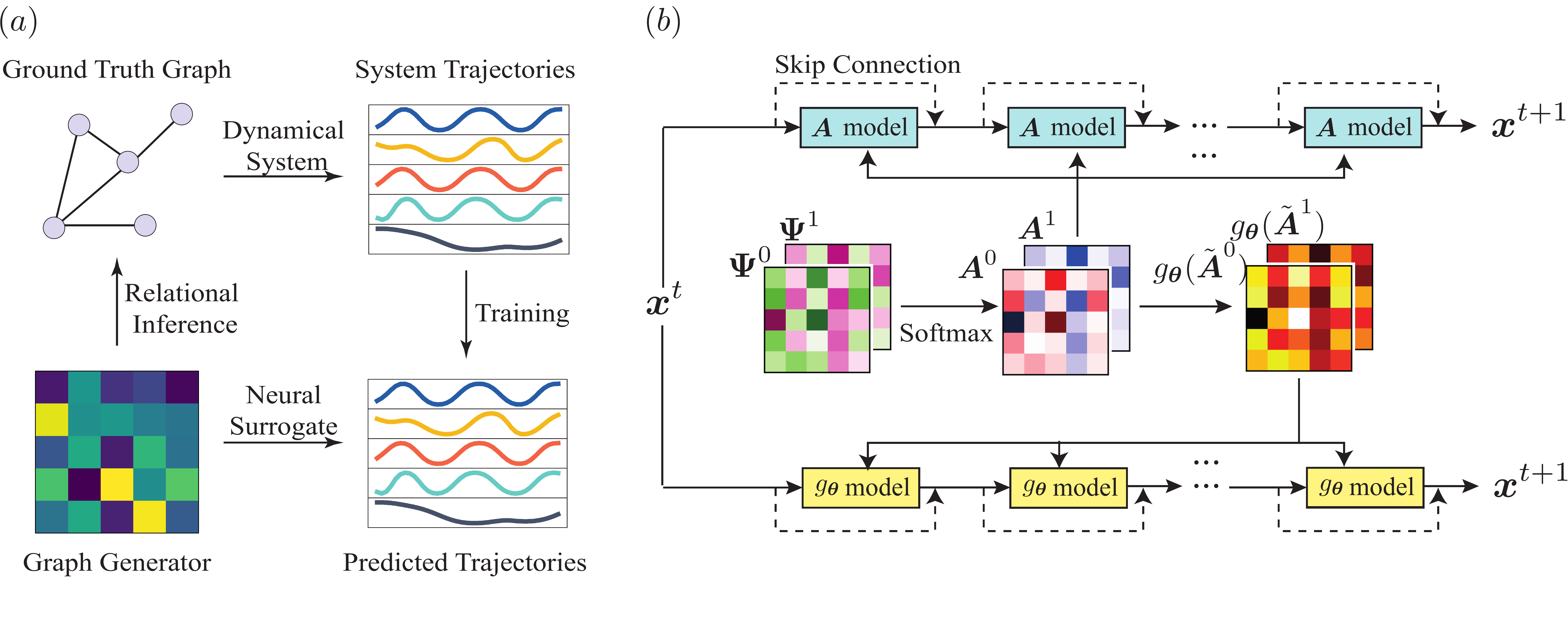}
    \caption{(a) An illustration of the relational inference problem. (b) The architecture of the full model. We train two dynamics surrogates with shared topology. The detailed architecture of $\mA$ model and $g_{\vtheta}$ model can be found in Section~\ref{sec:model}.}
    \label{fig:illustrate}
\end{figure*}
In this work, we propose a \textit{graph dynamics prior} (GDP) for relational inference. The ``prior'' terminology is an analogy with the \textit{deep image prior} used in inverse problems in image processing~\cite{ulyanov2018deep}, where the inductive bias of model (a convolutional neural network) steers reconstruction towards images with good properties, which is an implicit prior. Our model simultaneously uses a high-degree non-local polynomial and a ``shallow'' adjacency matrix to approximate the effective interactions between consecutive state samples. The polynomial filter is sensitive to graph perturbations which helps avoid poor local minima. As there are, in general, multiple graph matrices that can result in the same polynomial filters, simultaneously fitting a parallel single-step model resolves the direct interactions without converging to poor local minima thanks to the gradients from the polynomial model; this resolves issues which traditionally prevented the use of polynomial filters in NRI. Experiments show that GDP achieves significantly higher inference accuracy than any of the earlier approaches. Notably, it finds the direct interactions even at very low sampling rates where earlier approaches severely degrade.

\subsection{Related Work}

\noindent \textbf{Relational Inference}. Classical approaches to relational inference (RI) measure correlations~\cite{peng2009partial}, mutual information~\cite{wu2020discovering}, transfer entropy~\cite{schreiber2000measuring} or causality~\cite{quinn2011estimating} of system trajectories. These approaches do not perform future state predictions. Other studies focus on RI with known dynamical models~\cite{wang2016data,pouget2015inferring}. These are often designed for particular dynamics.
When only system trajectories are observed, NRI~\cite{kipf2018neural} infers the interaction graph in an unsupervised way while simultaneously predicting the state evolution. The scope of NRI has been extended to dynamic graphs~\cite{graber2020dynamic}, graphs with heterogeneous interactions~\cite{ha2023learning}, and modular meta-learning problems~\cite{alet2019neural}. It has also found applications in learning protein interactions~\cite{zhu2022neural} and was adapted to make causal claims~\cite{lowe2022amortized}.
The model has been extended to the non-amortized setting, where the graph encoder is often removed~\cite{lowe2022amortized,zhang2022universal}.
The dynamics learner part of these models is a GNN with one-step message passing.

\noindent \textbf{Spectral Graph Neural Networks}.  
Two basic design paradigms for GNNs are via graph (spatial) and spectral graph convolutions. ``Spatial'' graph convolution aggregates neighborhood information; spectral filters are generically non-local \cite{ortega2018graph}. 
ChebyNet~\cite{defferrard2016convolutional} parameterizes the convolution kernel by Chebyshev polynomials of the diagonal matrix of Laplacian eigenvalues. Two other GNNs that use polynomial graph filters are APPNP~\cite{gasteiger2018predict} and GPR-GNN~\cite{chien2021adaptive}. By making the weights of matrix polynomials trainable, GPR-GNN adaptively implements low-pass or high-pass graph filters.
For relational inference, spatial GNNs more straightforwardly preserve locality via one-step message-passing and thus have been widely considered in different methods.

\subsection{Our Contributions}

Contrary to prior belief, we show that RI by non-localized filters can work much better than shallow alternatives. While earlier work avoids direct--indirect confusion via one-step message-passing, it suffers from local minima even in the presence of weak indirect interactions which prevents it from fitting the dynamics. We conjecture and empirically demonstrate that a non-local model resolves this issue, provided that it is properly designed. Concretely, to mitigate the ambiguities arising from rooting matrix polynomials which prevented earlier uses of multi-step architectures, we run in parallel a local, single-step model with a shared adjacency matrix, but now benefiting from the ``steering'' by the multi-step model. This effectively results in a multi (two)-scale architecture. In addition to yielding better graphs, this strategy yields a much better model for the dynamics and greatly reduces the number of samples required to learn the graph. The two-scale architecture brings a remarkable performance improvement across the board.

\addtolength{\tabcolsep}{2pt}   
\begin{table*}[t]

\centering
\fontsize{9pt}{10pt}\selectfont 
\begin{tabular}{@{}lcccccccc@{}}
\toprule
Model & Graph & $\delta t$ & Volume &  MI & TE & NRI & GDP\\ 
\midrule
\textbf{Michaelis} &ER-50  & 1& $50\times 10$ &77.84&53.66&54.09±2.22& \textbf{98.31±1.41}\\
\textbf{Menten} &  &4  &$50\times 10$&55.93&51.95&51.85±0.97& \textbf{88.66±8.73} \\
 &BA-50  & 1& $50\times 10$ &88.04&63.96&55.20±1.81& \textbf{93.02±3.94}\\
 &  &4  &$50\times 10$&50.18&60.30&52.26±1.23& \textbf{87.42±6.08}\\
\midrule
\textbf{Rössler} &ER-50  & 1& $50\times 10$ &50.65&54.17 &60.35±4.58& \textbf{99.89±0.23}\\
\textbf{Oscillators} &  &4  &$50\times 10$&52.28&54.13& 51.95±1.28& \textbf{56.82±3.56}\\
 &BA-50  & 1& $50\times 10$ &56.28&62.64& 59.81±5.74& \textbf{90.55±16.13}\\
 &  &4  &$50\times 10$&50.46&52.63 & 52.42±1.76& \textbf{59.22±5.24}\\
\midrule
\textbf{Diffusion} & ER-50 & $1$ & $20\times 10$  &56.00&57.63&70.66±7.80 & \textbf{93.44±4.87}\\
 &  & $4$ & $20\times 10$    & 71.28&68.99& 60.23±6.73& \textbf{93.39±4.87} \\
 & BA-50 & $1$ & $20\times 10$  & 72.06&61.71& 72.03±11.62& \textbf{94.41±3.23}\\
 &  & $4$ & $20\times 10$  &86.38&69.77&59.01±5.73& \textbf{90.55±5.14} \\
\midrule
\textbf{Spring} & ER-50 & $20$ & $15\times 10$ &  72.24&76.05    & 99.84±0.47 & \textbf{99.99±0.02} \\
 &  & $60$ & $15\times 10$  & 71.43&69.17   &  97.47±2.93 & \textbf{98.96±1.25}\\
 & BA-50 & $20$ & $15\times 10$  & 91.16&84.67    & 98.17±5.40 & \textbf{99.88±0.36}\\
 &  & $40$ & $15 \times 10$ &\textbf{92.82}&63.67    & 67.89±9.89& 83.77±9.14 \\
\midrule
\textbf{Kuramoto} & ER-50 & $1$ & $30\times 30$ &64.69&64.76& 82.09±19.14 & \textbf{94.93±12.94}\\
 &  & $4$ & $30\times 30$   & 75.34&63.53&  95.96±5.01& \textbf{99.30±1.67}\\
 & BA-50 & $1$ & $20\times 30$   &55.46&61.87& 69.70±18.16& \textbf{90.13±12.38}\\
 &  & $4$ & $20\times 30$  &51.13&64.62& 89.57±11.69& \textbf{97.48±2.78}\\
 \midrule
\textbf{FJ} & ER-50 & $1$ & $20\times 10$  &53.66&83.64& 97.67±1.06 & \textbf{99.82±0.47}\\
 &  & $4$ & $20\times 10$   &58.98&59.00& 65.25±11.98& \textbf{75.48±10.60}\\
 & BA-50 & $1$ & $20\times 10$   &52.32&86.88&  91.62±4.67& \textbf{92.63±13.46}\\
 &  & $4$ & $20\times 10$    &50.90&67.13&67.27±9.59 & \textbf{73.89±9.84}\\
\midrule
\textbf{CMN} & ER-50 & $1$ & $20\times 10$   & 87.39&64.35& 89.76±2.59 & \textbf{97.58±3.38}\\
 &  & $4$ & $20\times 10$    &93.13&74.08& 89.94±1.42 & \textbf{98.40±1.92}\\
 & BA-50 & $1$ & $20\times 10$   & 87.84&71.51& 83.35±2.30 & \textbf{88.83±6.19} \\
 &  & $4$ & $20\times 10$   &92.28&75.39& 83.05±2.35& \textbf{92.97±5.26} \\ 
 \midrule
\textbf{Netsim} & -- & $1$ & $5\times 200$   & 94.73&74.83& 71.57±1.57& \textbf{95.09±0.68} \\
 & & $2$ & $5\times 100$   &94.10&50.20&65.90±6.24& \textbf{94.70±0.16} \\
\bottomrule
\end{tabular}
\caption{Interaction graph inference accuracy measured by AUC and for various dynamical systems and inferring models. The results for NRI and GDP are averaged over $10$ independent runs.  ER-$n$ or BA-$n$ denotes the name number of nodes ($n$) of the graph and $\delta t$ denotes the sampling interval.  The sampling interval refers to second-time samplings in pre-generated trajectories, not samplings during solving the ODEs. In the VOLUME column, $a\times b$ corresponds to \#trajectories $\times$ \#sampled steps. Boldface marks the highest accuracy.\label{tab:res}}
\end{table*}
\addtolength{\tabcolsep}{-2pt}
\section{Preliminaries}

We consider a graph $G = (V,E)$ on $|V| = n$ vertices, which describes the interaction relations among components of a dynamical system. Let $\mA$ be the adjacency matrix, $\tilde{\mA}=\mD^{-1/2} \mA \mD^{-1/2}$ its symmetric normalized version, and $\tilde{\mL} = \mI - \mD^{-1/2} \mA \mD^{-1/2}$ the symmetric normalized Laplacian. We also use $\mM$ to denote a general graph matrix, either the adjacency matrix or the Laplacian. 

\subsection{Graph Dynamical Systems and Relational Inference}

Node $i$ at time $t$ is described by a state vector $\vx_i^t \in \R^{d_s}$. We write $\vx^t=(\vx_1^t,\vx_2^t,\cdots,\vx_n^t)$ for the state of all nodes at time $t$. We consider both continuous- and discrete-time graph dynamical systems with synchronized updates,
\begin{equation*}
    \begin{array}{cc}
        \textbf{Continuous} & \textbf{Discrete} \\
        \dot{\vx}_i^t = f_i(\vx_i^t, (\vx_k^t)_{k\in N_i} ) &  \vx^{t+1}_i = f_i(\vx_i^t, (\vx_k^t)_{k\in N_i}),
    \end{array}
\end{equation*}
where $N_i$ is the set of vertex $i$'s neighbours and the dot denotes the time derivative. In relational inference, we observe snapshots of a graph dynamical system $\{\vx^t, t\in T\}$ at some set $T$ of observation times and aim to find the interaction graph from these snapshots, without any knowledge about the form of $f_i$. Figure~1 (a) illustrates the common framework for neural network-based RI approach. While the ground truth graph and the dynamics (top row) are both unknown, a graph generator and a dynamics surrogate are trained simultaneously (bottom row). The unknown graph is then predicted to be the best graph that explains the snapshots $\{\vx^t,t\in T\}$.

\subsection{Polynomial Graph Filters}

A degree-$K$ polynomial graph filter with coefficients $\vtheta = (\theta_0, \ldots, \theta_K)$ is defined as
$
g_{\vtheta} (\mM) = \sum_{k=0}^{K} \theta_k {\mM}^k.
$
Assuming $\mM$ is symmetric,\footnote{This assumption is made for simplicity of explanation; our proposed method easily handles directed graphs.} we let $\mM = \mU \mLambda {\mU}^\top$ be its eigenvalue decomposition. Then $g_{\vtheta} (\mM) = \mU g_{\vtheta}(\mLambda) {\mU}^\top$, where $g_{\vtheta}(\mLambda)$ applies element-wise to the diagonal elements of $\mLambda$ and $g_{\vtheta}(\lambda) = \sum_{k=0}^K \theta_k \lambda^k$. The scalar polynomial $g_{\vtheta}(\lambda)$ is called the convolution kernel. By the Weierstrass approximation theorem, any continuous function on a bounded interval can be approximated by a polynomial with arbitrary precision. 

\section{Interaction Retrieval }\label{sec:effective}

\subsection{Effective Interaction Graph}

Unless the sampling rate is very high, the effective graph modeled by the neural surrogate contains both direct and indirect interactions. Let us illustrate this. Consider a linear dynamics with scalar states. Let $\vx^t = [x_1^t, \ldots, x_n^t]^T$ be the vector of node states at time $t$ and  
$
    \dot{\vx}^t = \beta \mM \vx^t.
$
Solving the linear ODE, the states at $t$ and $t+\delta t$ are related as
$
    \vx ^{t+\delta t}= \exp(\beta \mM \delta t) \vx^t,
$
where $\exp(\beta \mM \delta t)$ is the matrix exponential which encodes the effective interaction graph. Since
\[
\exp(\beta \mM \delta t) = \mI + \beta \mM \delta t + \frac{\beta^2}{2!} \mM^2 {\delta t}^2+ \cdots,
\]
the interactions in principle exist instantaneously for all path-reachable nodes. For small $\delta t$, the power series can be approximated by truncating at first order in $\delta t$ and the interaction graph is effectively encoded by $\mM$, but for moderate or large $\delta t$, we need to include the higher-order terms.

\begin{figure}[t]
\centering
\includegraphics[width=0.9\linewidth]{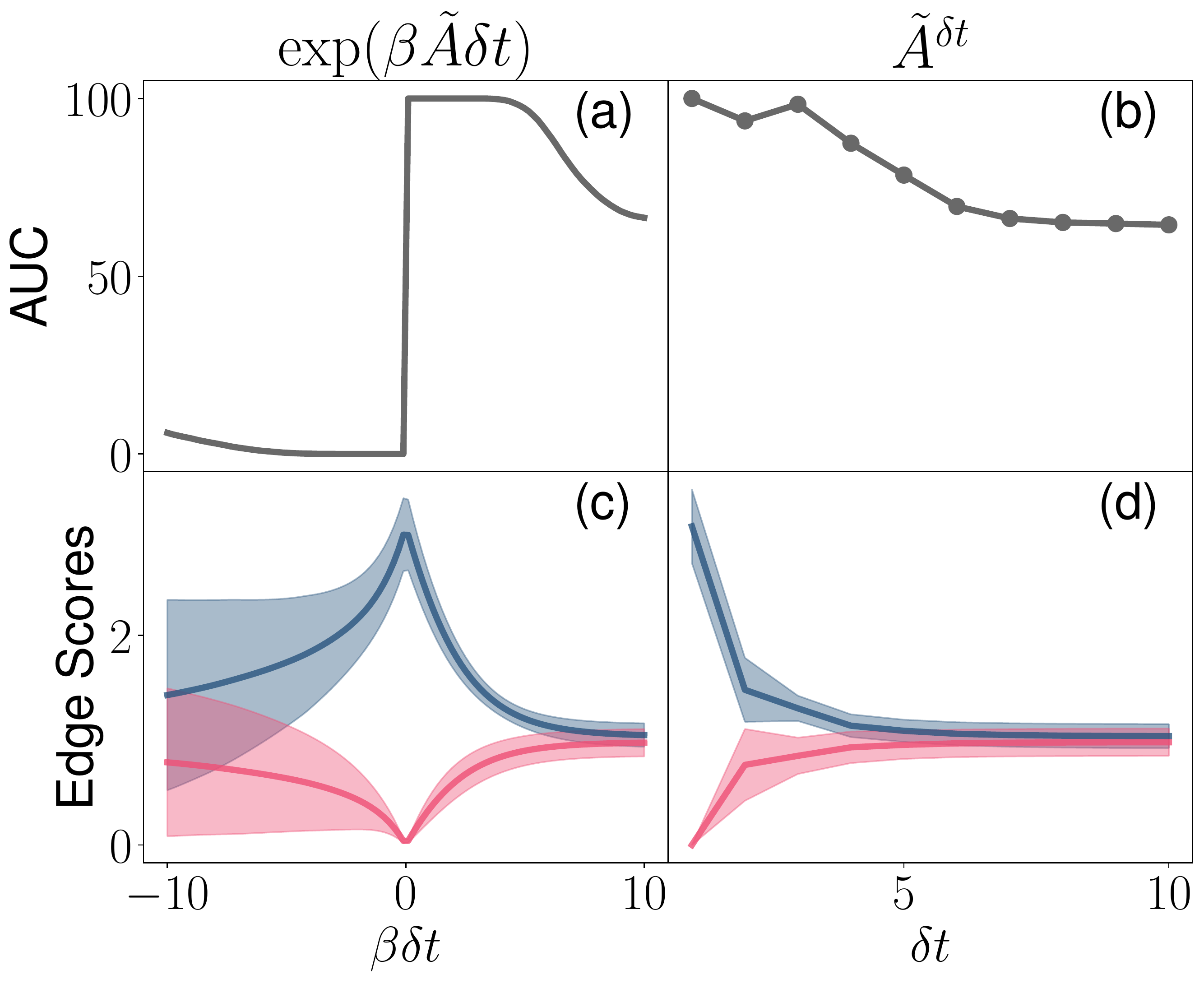}
\caption{Effects of observation intervals on the effective interaction graph. The AUC for (a) continuous and (b) discrete dynamics. The absolute value of normalized average score for (c) continuous and (d) discrete dynamics. The shaded regions show the standard deviation of each edge class. The results are obtained on an Erdös--Rényi graph with $n=30$ and $p=0.3$.}
\label{fig:poly}
\end{figure}

For general nonlinear dynamics, let $\bar{N}_i$ be the set of nodes that $x^{t+\delta t}_i$ \emph{effectively} depends on. Let $q_i$ be the effective transition function,
$
    x^{t+\delta t}_i = q_i((x_k^t)_{k\in \bar{N}_i}).
$
If $q_i$ is continuous, Kolmogorov--Arnold  theorem lets us write~\cite{khesin2014arnold,zaheer2017deep}
$
x_i^{t+\delta t} = \rho_i ( \sum_{j\in \bar{N}_i} J_{i,j} \phi(x_j^t) ),
$
for some continuous $\phi$ and $\rho_i$. The function $\phi$ is independent of $q_i$, and the parameters $\mJ = (J_{i,j})$ can be interpreted as interaction strengths. To build a neural surrogate for the effective system, we can approximate $\phi$ and $\rho_i$ by neural networks and identify the interaction strengths via training; the NRI decoder can be interpreted in this sense. NRI further distinguishes a node and its neighbours as $
x_i^{t+\delta t} = x_i^{t} + \rho (x_i, \sum_{j\in N_i} A_{i,j} \phi(x_i^{t}, x_j^t) )  
$, where $\rho$ and $\phi$ are neural networks and $\mA = (A_{i,j})$ trainable parameters.

While neural networks may approximate effective dynamics, the inferred matrix $\mA$ is (at best) close to the \textit{effective} interaction graph $\mJ$, rather than the true adjacency. In a linear system the effective graph is generated by a polynomial of the transition matrix which motivates the following \textit{polynomial} neural surrogate:
\[
x_i^{t+\delta t} = x_i^{t} + \rho \bigg(x_i, \sum_{j\in N_i} g_{\vtheta}(\mA)_{ij} \phi(x_i^{t}, x_j^t) \bigg).  
\]
We assume the dynamics to be invariant to the neighbour permutations, and therefore let $\rho$ be node-independent~\cite{zaheer2017deep}.
Further, although the above form of the Kolmogorov--Arnold theorem considers scalar node states, we use it as a heuristic to motivate the functional form of the effective interaction graph in GDP for both scalar and vector states. We will experimentally show that this surrogate benefits RI in both linear and nonlinear cases in Section 5. In the following, we discuss (\romannumeral1) when the effective interaction graph can confuse direct and indirect interactions, and (\romannumeral2) when and how a polynomial neural surrogate can help to set things straight.

\subsection{Effect of Observation Intervals}

To what extent the effective interaction graph reflects the direct interactions is determined by the intrinsic properties of dynamics and by the sampling interval $\delta t$. In this section, we empirically study the effect of the sampling rate in two cases where the effective graph can be determined exactly: continuous-time linear dynamics $\dot{\vx}^t = \beta \tilde{\mA} \vx^t$ where the effective interaction graph is $\mJ = \exp(\beta \tilde{\mA} \delta t)$, and a discrete-time linear system with synchronized updates $\vx^{t+1}= \tilde{\mA} \vx^t$ where the effective graph is $\tilde{\mA}^{\delta t}$. We compare the true graph defined by $\mA$ with the effective graph $\mJ$ by plotting the AUC as a function of $\delta t$ in Figure~\ref{fig:poly}. 


There is a qualitative distinction between the continuous and discrete systems, but in both performance deteriorates as $\delta t$ grows.
For continuous dynamics, the AUC remains close to 100\% for small $\beta \delta t$, meaning that the effective graph is dominated by direct interactions. When $\beta \delta t$ becomes larger, direct--indirect confusion deteriorates performance. In fact, even for small $\beta \delta t$ the average scores for positive and negative classes become closer. This signifies a loss of stability which makes it more likely to misclassify edges.

For discrete dynamics, odd hops result in a larger AUC, which leads to the perhaps counter-intuitive conclusion that larger observation intervals do not always yield worse RI. Intuitively, there is always a length-$3$ walk between two directly connected nodes $i,j$ as $i\to j\to i \to j$ , but we can find a length-$2$ walk between only when they share a common neighbour. The general trend is still that large observation intervals result in a direct--indirect confusion.

From the examples, we can classify sampled dynamics into (i) those with effective interaction graphs closely mirroring direct interactions and (ii) those where coarse sampling weakens this correlation. We use the terms ``strong'' and ``weak'' correlations informally, without a strict boundary between the two. In Section~\ref{sec:exp}, we demonstrate that many known dynamical models fit the first category. Finally, we emphasize that unlike our proposed two-scale polynomial architecture, simply stacking multiple message-passing layers does not resolve undersampling and might even worsen performance; we show this in Appendix C.1.

\subsection{Interaction Graph Retrieval and Noise Amplifier Effect of Graph Polynomials}

Even if we can find the \textit{effective} interaction graph (which is the best we can hope for without additional assumptions) and this graph is correlated with the true adjacency, the question is how to recover direct interactions. We discuss this issue in cases where the effective graph is either strongly or weakly correlated with the direct interactions. 

\begin{figure}[t]
\centering
\includegraphics[width=0.9\linewidth]{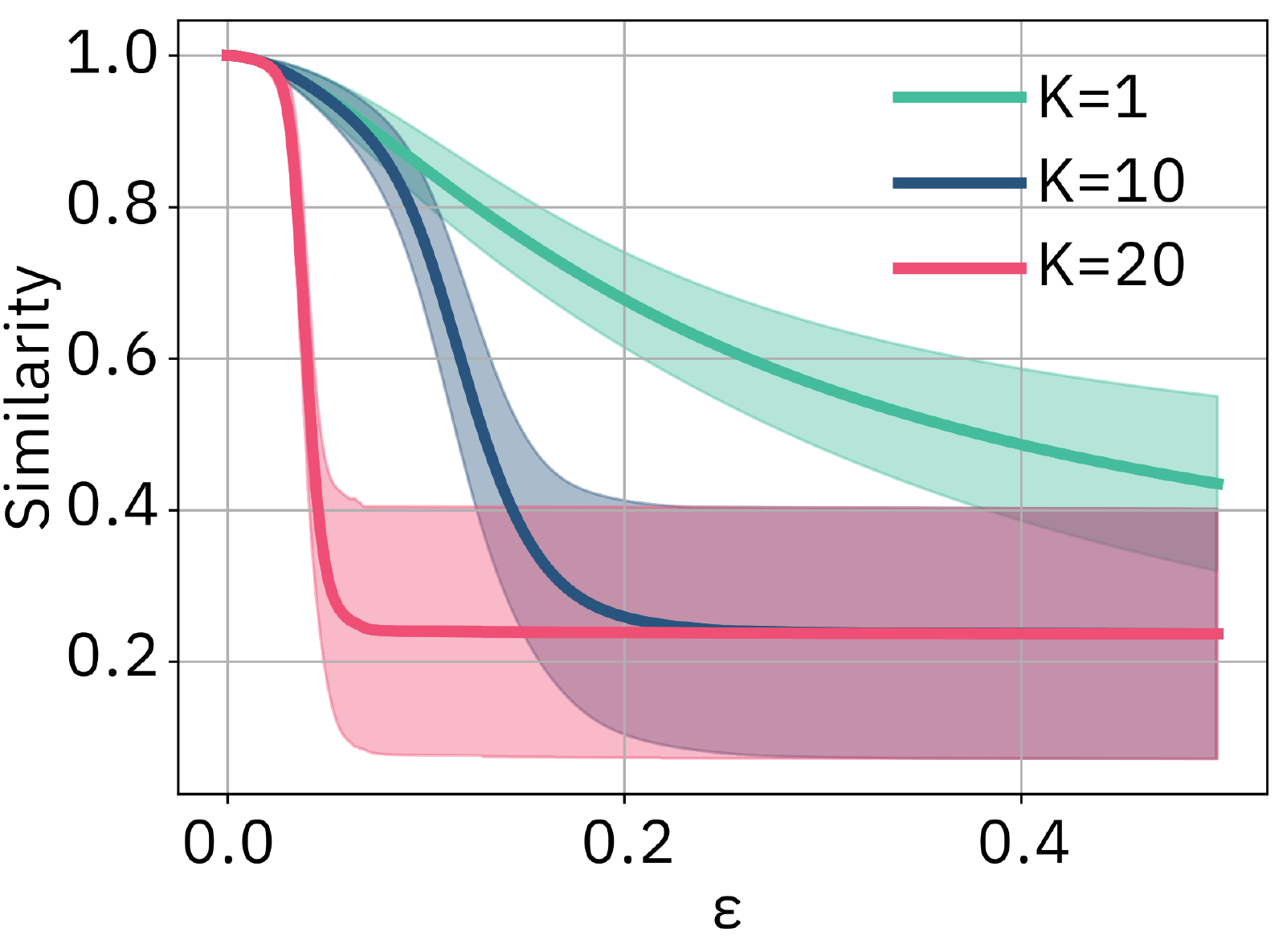}
\caption{The noise amplifier effect of graph polynomials for RI. The results are obtained on an Erdös--Rényi graph with $n=50$ and $p=0.1$. Other parameters are chosen to be $\vtheta = 1$ and $t=1\mathrm{e}{-5}$.}
\label{fig:stability}
\end{figure}

First, we consider a weakly correlated example. Consider a linear system in which the effective graph is a polynomial of $\mM$ with coefficients $\vtheta^\star$, $\mJ = g_{\vtheta^\star}(\mM)$. Suppose we find $\mJ$ and now want to find $\mM$; we thus need to solve $g_{\vtheta^{\prime}}(\mM^{\prime})\approx \mJ$ for $\mM^{\prime}$ and $\vtheta^{\prime}$ (note that $\mM^{\prime}$ does not necessarily equal $\mM$). But the solution of $g_{\vtheta^{\prime}}(\mM^{\prime})\approx \mJ$ is in general not unique even when we know the polynomial coefficients $\vtheta$. For example, if $\mJ = \mM^2$, the matrix square root equation $(\mM^{\prime})^2 = \mJ$ is solved by $\mU \mathrm{diag}(\mu_1,\cdots,\mu_n) \mU^\intercal$, where $\mu_i = \pm \sqrt{\lambda_i}$. In order to identify the correct sign pattern we need to use additional prior knowledge about the graph, such as sparsity.
In Lemma A.1 of Appendix~A.1, we show that for a general polynomial with known $\vtheta$ the solution is unique only when the convolution kernel $g_{\vtheta}(\lambda)$ is injective (for example, the (matrix) exponential) and the graph matrix has no repeated eigenvalues. When $\vtheta$ is not known, the number of solutions becomes much larger (we always have the trivial solution $\theta^{\prime}_i = \delta_{i,1}$ and $\mM^\prime = \mJ$.) 

In the strongly correlated case, the interaction graph is close to the direct interaction graph in the sense that positive and negative edges can be classified with non-trivial accuracy. While the equation $g_{\vtheta^{\prime}}(\mM^{\prime})\approx \mJ$ may still have many solutions we only want to build a binary classifier which is possible directly from the effective graph. This means that single-step message passing may suffice. However, as we show experimentally in Section~\ref{sec:exp}, even weak indirect interactions trap this architecture in poor local minima. We now discuss how error amplification in polynomial filters solves the problem.

For $
x_i^{t+\delta t} = x_i^{t} + \rho (x_i, \sum_{j\in N_i} M_{i,j} \phi(x_i^{t}, x_j^t) )  
$, changing the value of $M_{i,j}$ only affects the next state of node $i$,  $x_i^{t+\delta t}$. It means that errors in the learned interactions only generate gradients locally on the graph. But if we approximate the effective interaction graph by a polynomial, 
$
    g_{\vtheta} (\mM) = \sum_{k = 0}^K \theta_k \mM^k \approx \mJ,
$
then a perturbation in $M_{i,j}$ propagates through several hops, generating large loss for multiple nodes and thus removing local minima. The issue is that, while the above polynomial equation can always be solved (consider again the trivial assignment $\theta_k = \delta_{k,1}$, $\mM=\mJ$), the solution does not necessarily correspond to the direct interaction graph.

Before we show how to address this issue in Section~\ref{sec:model}, let us demonstrate experimentally how including higher-order terms can make the model more sensitive to errors in the graph and produce better gradients. Let $\mM_\epsilon = \mM + \epsilon \mXi$, where $\mXi$ is a perturbation to the graph matrix $\mM$. Consider the graph filter $\mM_\epsilon + t g_{\vtheta}(\mM_{\epsilon})$, and let $\vy_{t,\epsilon} = (\mM_\epsilon + t g_{\vtheta}(\mM_{\epsilon}))\vx$ be the filtered signal; $\vy_{0,0}=\mM \vx$ corresponds to the unperturbed case. We use cosine similarity $\cos(\vy_{0,0},\vy_{t,\epsilon})$ to quantify the effects of graph perturbations. In Appendix A.2, we show that in the case $\epsilon=0$, the cosine similarity is bounded from below as
$$
\cos(\vy_{0,0},\vy_{t,0}) \geq 1-t^2\vert O_N(1) \vert + O_t(t^3),
$$
which indicates that the cosine similarity is close to one with $t$ being sufficiently small. Then, we show numerically that when $\epsilon$ deviates from zero, the cosine similarity becomes significantly smaller. We consider uniform random noise $\mXi$ and plot the cosine similarity $\cos(\vy_{0,0},\vy_{t,\epsilon})$ versus $\epsilon$ in Figure~\ref{fig:stability}: the similarity decays rapidly when increasing $K$. The above results show that a polynomial graph filter can amplify the noise in the graph when measured through the filtered signals.

\section{GDP for Relational Inference}\label{sec:model}

\paragraph{Graph Generator} We now introduce the full architecture of GDP. We consider the non-amortized setting without an encoder~\cite{lowe2022amortized,zhang2022universal}. We use a simple generator where the probability of an edge $(i, j)$ is
\begin{equation}\label{eqn:generator}
    A^a_{i,j} = \left(\mathrm{Softmax} (\beta [\Psi_{i,j}^0, \Psi_{i,j}^1])\right)_a\ \text{for}\ a\in\{0,1\},
\end{equation}
$\Psi_{i,j}^0, \Psi_{i,j}^1 \in \R$ are trainable latent variables, and $\beta$ is the inverse temperature. Note that $A^1 = \vone \vone^\intercal - A^0$. Including both in the model (as below) significantly speeds up convergence.

\paragraph{Neural Dynamics Surrogate}

We generate two probability matrices, $\mA^{(a)}$ for $a\in\{0,1\}$ and the corresponding (probabilistic) graph filters $\mF^{(a)} = g_{\vtheta}(\tilde{\mA}^{(a)})$, where as before $\tilde{\mA}^{(a)}$ is the symmetric normalized version of $\mA^a$, and the filter coefficients $\vtheta$ are trainable. We use in-degree (for the direction of message-passing) to normalize the adjacency matrix in the directed graph case. The dynamics surrogate predicts the next state as
\begin{equation}\label{eqn:surrogate}
\begin{split}
\tilde{\vh}_{(i,j)}^t & = \sum_{a\in\{0,1\}} F^a_{i,j} \tilde{f}^a_e \left(\vx_i^t,\vx_j^t \right),\\ 
\vx_j^{t+1} & = \vx_j^t + \tilde{f}_v \left(\sum_{i\neq j} \tilde{\vh}^t_{i,j}\right).
\end{split}
\end{equation}
In the above architecture, $\tilde{f}_e^a$ are edge-wise MLPs, and $\tilde{f}_v$ is a vertex-wise MLP. 
As explained in the previous section, a polynomial filter removes the local minima generated by indirect interactions, but it cannot guarantee that the learned $\mA$ is close to direct interactions since the roots of the matrix polynomial are not unique.  Therefore, we simultaneously train another parallel dynamics neural surrogate using only the adjacency matrix, i.e., replacing  $F_{i,j}^a$ by $A^a_{i,j}$ in Equation~\eqref{eqn:surrogate}. The loss is simply the sum of MSEs of each neural surrogate. This strategy encourages the solution to stay close to the ground truth interactions while also leveraging the error amplification from graph polynomials. The overall architecture of the proposed model is shown in Figure~\ref{fig:illustrate}(b). We call the model in Equation~\ref{eqn:surrogate} using the adjacency matrix as $\mA$ model, and the model using a polynomial graph filter as $g_{\vtheta}$ model. The two models work in parallel with shared $\mA$. 
\begin{figure}[t]
\centering
\includegraphics[width=0.9\linewidth]{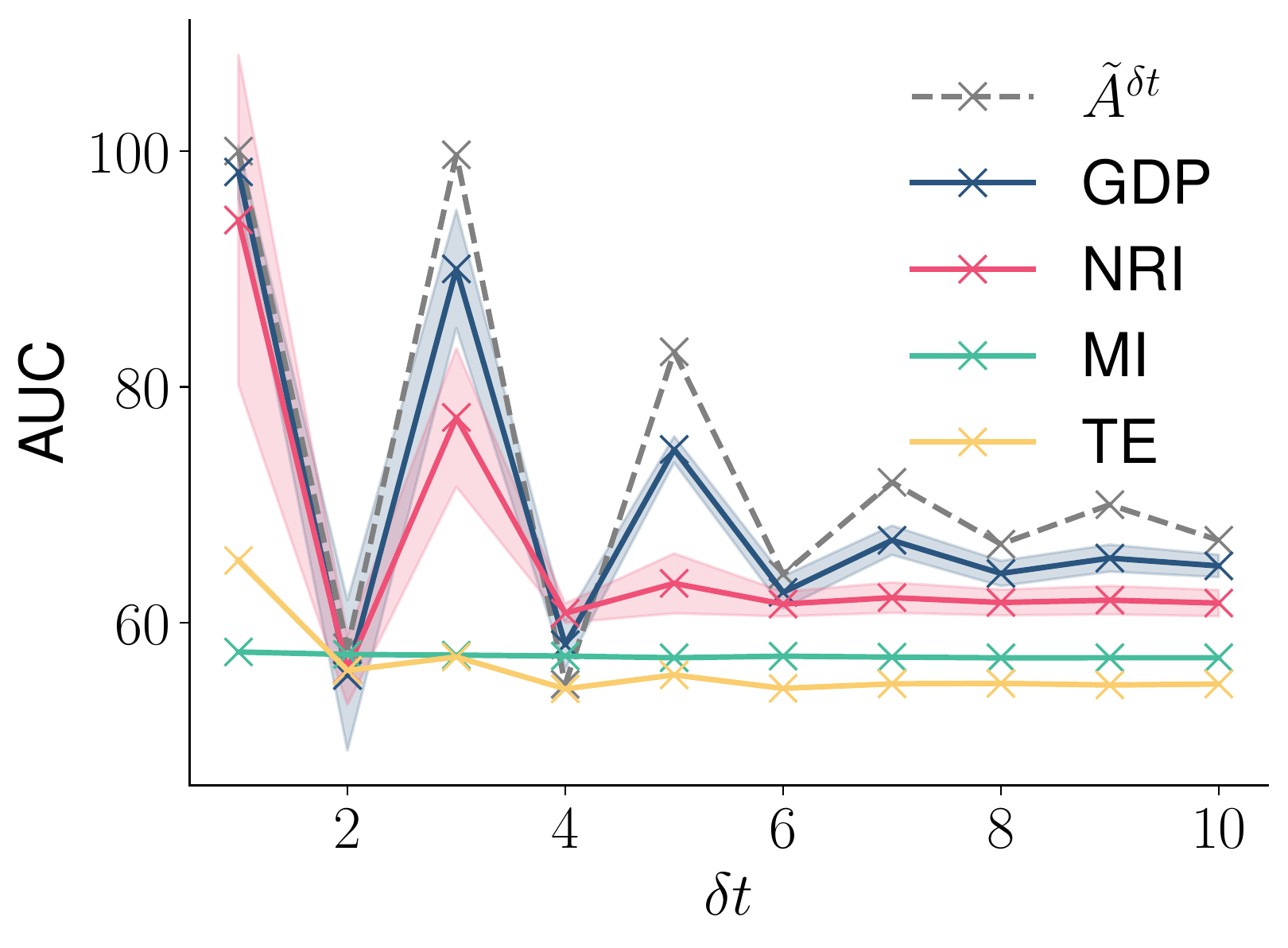}
\caption{AUC versus sampling rate $\delta t$. The gray dashed line denotes the results from explicit interaction graph. For all the experiments, we keep the volume of training data ($\#$trajectories $\times$ $\#$sampled steps) identical. The results are obtained in ER-30 with data volume $30\times 20$.}
\label{fig:interval}
\end{figure}
 
\subsection{Application to Stochastic Dynamics (fMRI)}

We further show that with appropriate modifications GDP can be adapted to work with stochastic systems such as functional brain region dynamics in fMRI. Indeed, with an addition of temporal smoothing it performs considerably better than a single-step model and as well as the state-of-the-art method based on mutual information, but unlike that method we learn a model for the dynamics. Using multiple $\mA$- or $g_{\vtheta}$-models and deeper GNNs further improves performance.
We leave a more detailed analysis of stochastic dynamics for future work.

\section{Experiments}{\label{sec:exp}}

We consider several representative graph dynamical systems, both continuous and discrete, to validate the proposed algorithm. The continuous-time systems include (\romannumeral1) the Michaelis--Menten kinetics~\cite{karlebach2008modelling}, a model for gene regulation circuits; (\romannumeral2) Rössler oscillators~\cite{rossler1976equation} on graphs, which can generate chaotic dynamics; (\romannumeral3) diffusion, which is a simple a continuous-time linear dynamics; (\romannumeral4) a network-of-springs model which describes particles interacting via Hooke's law; and (\romannumeral5) the Kuramoto model~\cite{kuramoto1975self} which is a network of phase-coupled oscillators. The discrete-time systems include (\romannumeral6) Friedkin-Johnsen dynamics~\cite{friedkin1990social}, a classical model for describing opinion formation~\cite{abebe2018opinion,okawa2022predicting}, polarization and filter bubble~\cite{chitra2020analyzing} in social networks; and
(\romannumeral7) the coupled map
network (CMN)~\cite{garcia2002coupled}, a discrete-time model with chaotic behavior. Moreover, we considered a publicly available fMRI dataset (\romannumeral8) Netsim~\cite{smith2011network}, comprising realistic simulated data. A more detailed description of the dynamics and and data generation details can be found in Appendix~B.1. The graphs in all datasets but Netsim are undirected. We include further experiments on directed graphs in Appendix C.5. Importantly, we also carry out experiments on real-world data in Appendix C.6 (a dataset of traffic information and a gene regulatory network of S. cerevisiae yeast), and analyze the impacts of graph topology on inference accuracy in Appendix C.7.

\subsection{Results on Relational Inference}

We compare GDP to several baselines. The first one is NRI. The original NRI is designed for the amortized setting where the trajectories do not share the underlying graph. As we primarily consider the classical non-amortized case, we use the version of NRI without a graph encoder~\cite{lowe2022amortized}. We further consider two statistical approaches based on mutual information (MI) and transfer entropy (TE). Implementation details for the baselines can be found in Appendix~B.2. The hyperparameters for GDP are summarized in Appendix~B.3. As Equation~\ref{eqn:surrogate} is invariant to swaps of edge types $a\in \{0,1\}$, it is possible that the learned graph corresponds to the complement graph of direct interactions. This ambiguity is discussed in Appendix~B.4.

We apply GDP to the simulated systems and compare it to the baselines. We conduct experiments on Erdős–Rényi (ER) and Barabási--Albert (BA) graphs of different sizes and measure the interaction recovery accuracy by AUC. 
Table~\ref{tab:res} summarizes the average AUC with standard deviations. The first four columns describe the dynamical model, graph type, sampling rate and data volume, respectively. As both NRI and GDP can reach higher accuracy with increasing data volume, the volume of data in Table~\ref{tab:res} is determined by the rough criterion that GDP reaches above $90$ AUC in the short sampling interval case. The data volumes are kept the same when increasing the sampling interval.

From the experiments, GDP significantly improves the baseline methods. For example, in the Michaelis--Menten model, GDP reaches good accuracy when the other baselines cannot recover helpful information (with an AUC of about $50.00$). 
GDP shows remarkable robustness to under-sampling. 
For example, for the Spring model in BA-50, while NRI and GDP generate accurate predictions with small sampling intervals, GDP degrades much less when increasing the sampling interval. The improvement is not limited to the large sampling rate case, as the error amplifier mechanism of the polynomial filter still works in this case. A phenomenon worth noticing is that the results generally display large fluctuations, which suggests that the loss landscape has many poor local minima. Using a polynomial filter helps escape these poor minima and improves the inference accuracy. In Appendix~C.2, we further study the dependence of model performance on the polynomial order $K$. In Appendix~C.3, we perform ablation studies to show that using only the polynomial filter is insufficient to generate stable predictions, as, in general, multiple graph matrices can result in the same polynomial filters. An interesting phenomenon is that a ``good'' graph for predicting the dynamics turns out to be close to the true graph. We further verify the that the true graph is a local attractor for our model in Appendix~C.4.

\begin{figure}[t]
\centering
\includegraphics[width=0.95\linewidth]{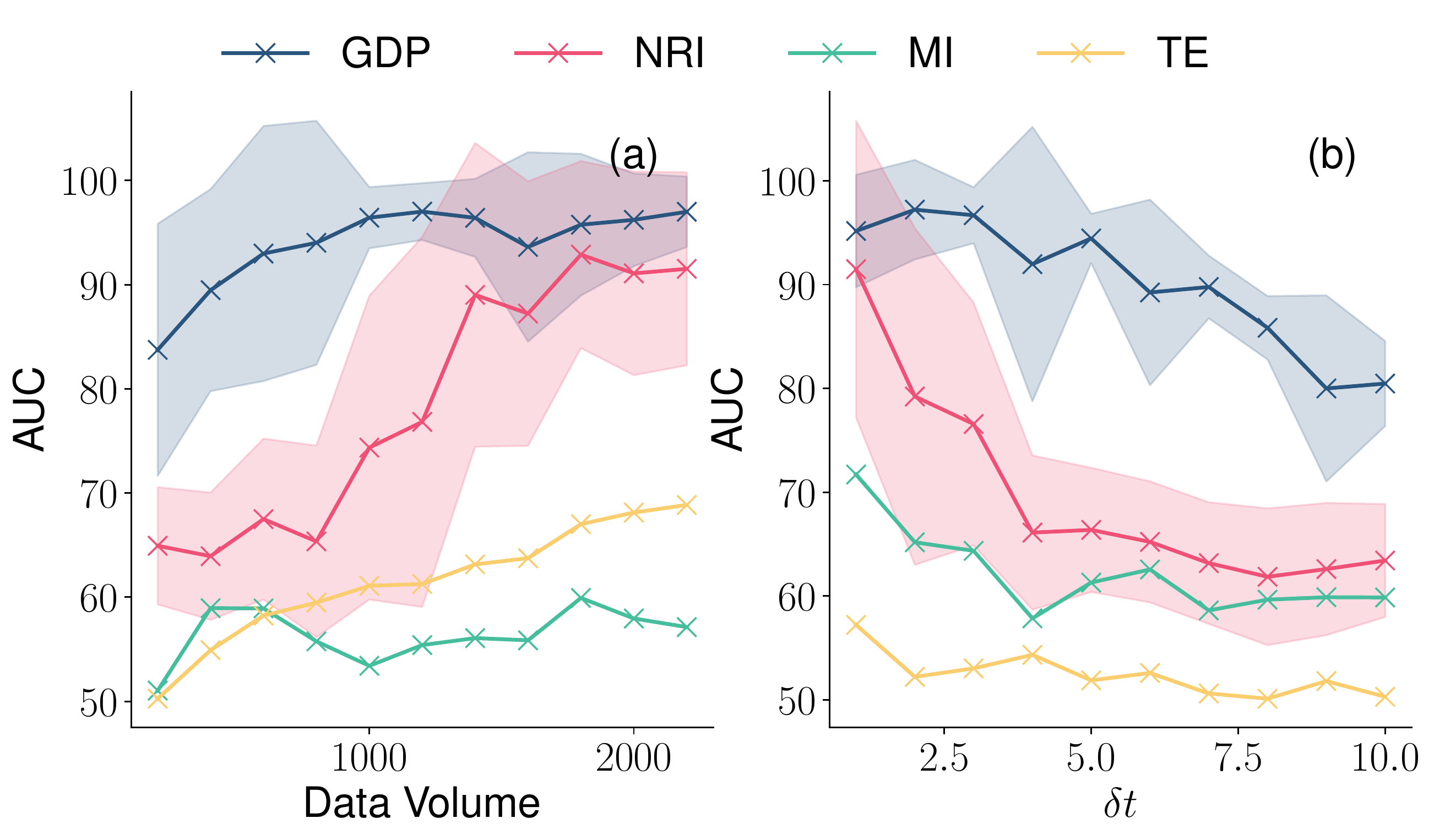}
    \caption{AUC versus (a) data volume and (b) sampling rate $\delta t$ . The experiments are performed on an ER-20 graph. In (a), the trajectory length and sampling interval are fixed as $10$ and $5$. We increase the data volume by including more trajectories. In (b), the data volume is fixed as of $50\times 10$. }
\label{fig:interval2}
\end{figure}

\subsection{Robustness to Sampling Rate and Data Volume}
\label{sec:interval}

We analyze the dependence of GDP's performance on sampling rates and data volume. For discrete-time linear dynamics $\vx^{t+1}=\tilde{\mA}\vx^t$, there is little correlation between the effective and direct interaction graph $\tilde{\mA}$ for even sampling intervals. This is confirmed using RI algorithms, depicted in Figure~\ref{fig:interval}. The gray dashed line displays results from $\tilde{\mA}^{\delta t}$. Predictably, even-interval sampling confuses GDP and the three baselines, preventing recovery of direct interactions; imposing constraints such as sparsity seems essential.

The results from $\tilde{\mA}^{\delta t}$ provide a rough upper bound for inference accuracy if we only use the effective interactions as the prediction. Both GDP and the other three baselines are below this line, which suggests that even when some positive edges are, in principle, distinguishable, these algorithms can still not find them. The discrepancy between the upper bound and algorithmic results is more pronounced when $\delta t$ is increased (for odd $\delta t$). When $\delta t=3$, the bound is $\approx 100$, close to the $\delta t=1$ case, but both NRI and GDP are further away in the former case. GDP nonetheless performs the best at all sampling rates.

We consider the Michaelis--Menten model for continuous-time dynamics. We increase $\delta t$ while keeping the volume of data fixed. The results are in Figure~\ref{fig:interval2} (b). For all algorithms the AUC decreases with $\delta t$, suggesting that direct and indirect interactions are more easily confused at larger sampling intervals. Still, GDP shows better robustness to the sampling rates. We next increase the number of training trajectories; Figure~\ref{fig:interval2} (a) shows the results. GDP performs best in all cases. It is the least sensitive to data volume and accurate even with small data sets. This may be essential in real applications where samples are hard or expensive to get.

\subsection{Polynomial Filters Help Escape Bad Local Minima}

\begin{figure}[t]
\centering
\includegraphics[width=0.765\linewidth, ]{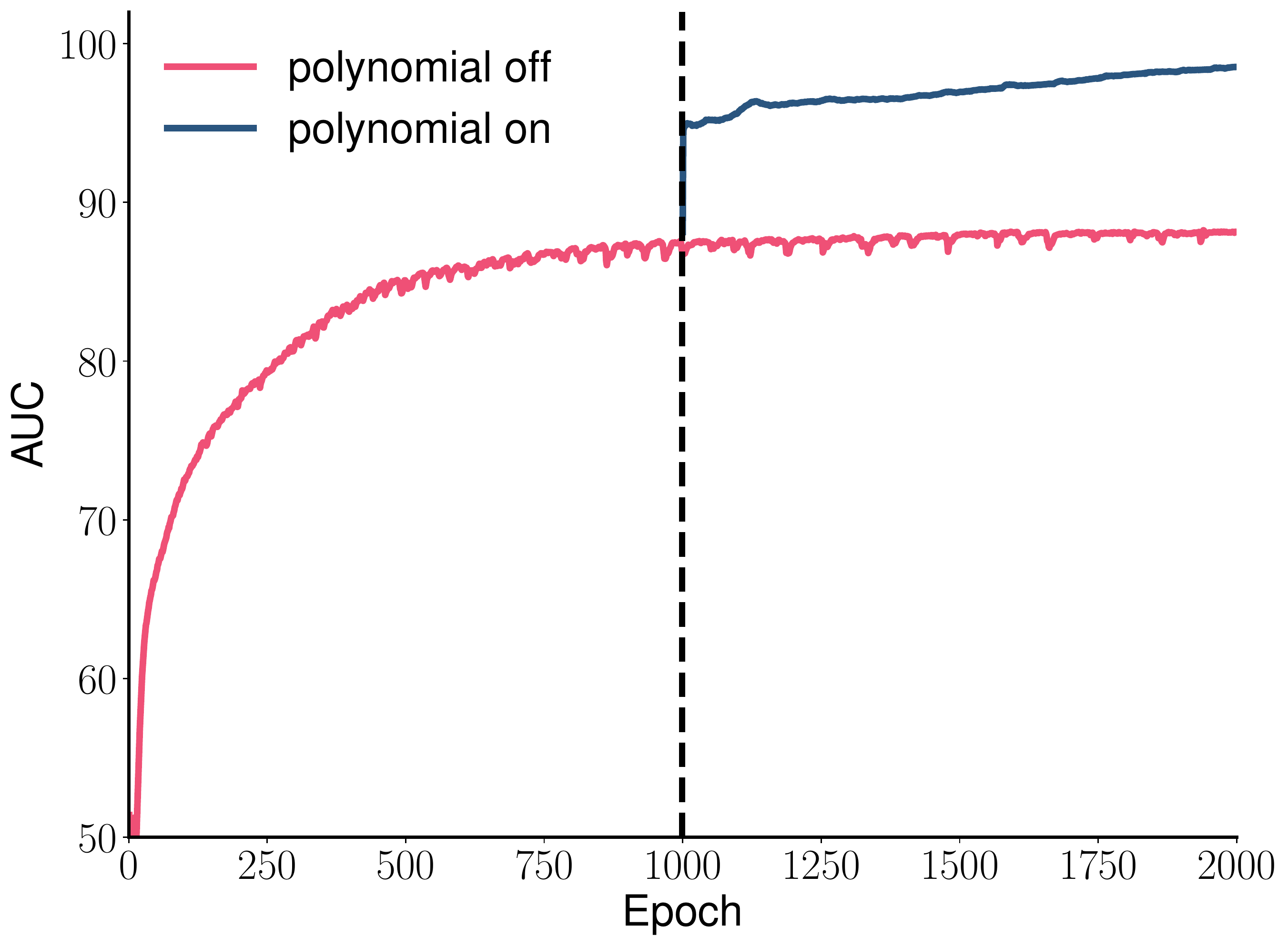}
\caption{Polynomial graph filters help to escape from local minima. The experiments are performed with Kuramoto model on an ER-50 graph.}
\label{fig:escape}
\end{figure}

We design experiments to provide empirical evidence that polynomial filters can help escape bad local minima.
We begin by training a model with only the one-step message passing part; once the AUC reaches a plateau we switch on the polynomial part. Figure~\ref{fig:escape} plots the evolution of the AUC over training epochs. The blue curve corresponds to when we insist on the one-step-only model: it stays approximately constant after $1000$ epochs. After activating the polynomial part, the AUC increases sharply in a single epoch, signaling that we immediately obtained a much more accurate graph. This phenomenon suggests that the polynomial filter indeed produces gradients that help escape the poor local minimum to which a one-step model converged.

\section{Discussion}

The experiments show that in a broad range of qualitatively diverse dynamical systems and for a broad range of sampling rates, the non-local neural surrogate in GDP indeed induces a favorable inductive bias and removes poor local minima that cause problems for the earlier local models. As a result, GDP achieves state-of-the-art performance across the board, often by a large margin and at much lower data volumes. We considered a setting where all trajectories share the same graph but our model can be extended to the amortized setting by using a graph encoder which preserves permutation invariance. The proposed model could also be adapted to make causal claims by using only the adjacency part at test time, as was done in  \cite{lowe2022amortized}.

GDP, and other NRI-type methods in general still have several limitations. Firstly,
although we have incorporated polynomial filters to address the non-local interactions induced by coarse temporal sampling, GDP works when the effective interaction graph is strongly correlated with the direct interaction graph. This seems to be the case for all combinations of dynamics and sampling rates we tested, but a more challenging task is to look at the strongly mixed case where there is only a weak correlation between the effective and the true graph. Secondly, we currently consider all nodes in the network to be observed; in practice, we only get partial observation whose topology may change with time. Thirdly, GDP may not be suitable in situations where the nodes are highly heterogeneous as in, for example, some metabolic networks. Finally, our claims and understanding of the method are currently based on heuristics and experiment; a precise theory is yet to be worked out.

\section*{Acknowledgments}
LP would like to acknowledge support from National Natural Science Foundation of China under Grand No.~62006122 and 42230406. CS and ID were supported by the European Research Council (ERC) Starting Grant 852821---SWING.

\bibliography{aaai24}

\appendix

\setcounter{table}{0}
\renewcommand{\thetable}{A\arabic{table}}
\setcounter{figure}{0}
\renewcommand{\thefigure}{A\arabic{figure}}
\newpage

\noindent{\Large \textbf{Technical Appendix}}

\section{Some Proofs}
\subsection{Graph Retrieval}\label{sec:retrieval}
\begin{lemma}
Let $g_{\vtheta}:\mathbb{R}\to\mathbb{R}$ be a polynomial convolution kernel,  $s_i$ the number of real roots of $g_{{\vtheta}}(x) - g_{\vtheta}(\lambda_i) = 0$, and $\mM$ a symmetric graph matrix. If (\romannumeral1) $\mM$ has distinct nonzero eigenvalues $\{\lambda_1, \lambda_2 \cdots , \lambda_n\}$ and (\romannumeral2) the convolution kernel $g_{\theta}$ is injective, i.e., $s_i=1$ for all $\lambda_i$, and $g_{\vtheta}(x)\neq 0$ for $x\neq 0$, then the matrix equation $g_{\vtheta}(\mM^\prime) = g_{\vtheta}(\mM)$ has a unique solution $\mM^\prime=\mM$. Otherwise, it has at least $\prod_{i=1}^n s_i$ solutions.  
\end{lemma}

\begin{proof}
Since $g_{\vtheta}(\mM^\prime)$ is polynomial,  $\mM^\prime$ and $g_{\vtheta}(\mM^\prime)$ commute; therefore $\mM^\prime$ and $g_{\vtheta}(\mM)$ commute since $g_{\vtheta}(\mM^\prime) = g_{\vtheta}(\mM)$. Thus $\mM^\prime$ and $g_{\vtheta}(\mM)$ are simultaneously diagonalizable. When $\mM$ has distinct nonzero eigenvalues and $g_{\vtheta}$ is injective and $g_{\theta}(x)\neq 0$ for $x\neq 0$, then also $g(\mM)$ has distinct nonzero eigenvalues. Then the basis in which $\mM^{\prime}$ and $g(\mM)$ are simultaneously diagonal is unique (up to permutations) and equal to the eigenbasis $\mU$ of $\mM$.  Denote the correspondingly diagonalized versions of $\mM$ and $\mM^\prime$ by $\mLambda$ and $\mLambda^{\prime}$, respectively.  
Then we have that $\mU g_{\vtheta}(\mLambda) \mU^{\intercal} = \mU g_{\vtheta}(\mLambda^{\prime}) \mU^{\intercal}$. Equating the diagonal entries gives 
$
g_{\vtheta}(\lambda_i) - g_{\vtheta}(\lambda^{\prime}_i) = 0.
$
If $g_{\vtheta}$ is injective, this implies $\lambda_i = \lambda_i^\prime$, and the solution is unique with $\mA=\mA^{\prime}$.

If $g_{\vtheta}$ is not injective, or the graph matrix $\mM$ itself has repeated eigenvalues, then $g_{\vtheta}(\mM)$ may have repeated eigenvalues and there may be multiple diagonalizing bases (with the eigenbasis $\mU$ being one of them).
For the eigenbasis, the equation $
g(\lambda_i) - g(\lambda_i^\prime) = 0
$ has $s_i$ solutions. Picking any of the $s_i$ solutions for each $i$ gives a solution to the matrix equation $g(\mM^\prime) = g(\mM)$, so there are $\prod_{i=1}^n s_i$ solutions. Therefore, there are at least $\prod_{i=1}^n s_i$ solutions.
\end{proof}

\subsection{Stability to Graph Perturbations}
Let $\mM_\epsilon = \mM + \epsilon \mXi$, where $\mXi$ is a perturbation to the graph matrix $\mM$. We use the perturbed matrix to define the graph filter $\mM_\epsilon + t g_{\vtheta}(\mM_{\epsilon})$. Let $\vy_{t,\epsilon} = (\mM_\epsilon + t g_{\vtheta}(\mM_{\epsilon}))\vx$ be the filtered signal. Note that $\vy_{0,0}=\mM \vx$ correspond to the unperturbed case. We can write $\vy_{t,\epsilon} = \vy_{0,0}+\Delta \vy_{t,\epsilon}$, where $\Delta \vy_{t,\epsilon} = \epsilon \mXi \vx + t g_{\vtheta} (\mM_{\epsilon})\vx$. Therefore, the cosine similarity $\cos(\vy_{t,\epsilon}, \vy_{0,0})$ depends on the competition between $\vy_{0,0}$ and $\Delta \vy_{t,\epsilon}$. The noise amplifier effect corresponds to the following scenario: with $t$ being sufficiently small, when $\epsilon = 0$ the cosine similarity is close to one; while when $\epsilon$ deviates from zero, the cosine similarity becomes significantly smaller. We show that the polynomial graph filter does help to achieve this effect. We prove the following result:

\begin{lemma}
When $\epsilon=0$, the cosine similarity between $\vy_{0,0}$ and $\vy_{t,0}$ is bounded from below as 
$$
\cos(\vy_{0,0},\vy_{t,0}) \geq 1-t^2\vert O_N(1) \vert + O_t(t^3).
$$
\end{lemma}
\begin{proof}
By definition
$$
\cos(\vy_{0,0},\vy_{t,\epsilon}) = \frac{\vx^\intercal \mM(\mM_{\epsilon} + t g_{\vtheta}(\mM_{\epsilon})\vx}{\Vert \mM \vx \Vert \Vert (\mM_{\epsilon} + t g_{\vtheta}(\mM_{\epsilon})) \vx\Vert}.
$$
The second norm in the denominator can be written as
\begin{equation*}
    \begin{split}
&\Vert (\mM_{\epsilon} + t g_{\vtheta}(\mM_{\epsilon})) \vx\Vert^2 \\
&= \Vert \mM_{\epsilon} \vx \Vert^2 + 2 t \vx^\intercal \mM_{\epsilon} g_{\vtheta}(\mM_{\epsilon}) \vx + t^2\Vert g_{\vtheta}(\mM_{\epsilon}) \vx \Vert^2.        
    \end{split}
\end{equation*}
To lighten the notation, we denote 
$$
E = \frac{\vx^\intercal \mM_{\epsilon} g_{\vtheta}(\mM_{\epsilon}) \vx}{\Vert \mM_{\epsilon} \vx \Vert^2},\ F = \frac{\Vert g_{\vtheta}(\mM_{\epsilon}) \vx \Vert^2 }{\Vert \mM_{\epsilon} \vx \Vert^2}
$$
and $C_{\epsilon} = \cos(\vy_{0,0},\vy_{0,\epsilon})$. We expand the cosine similarity when $t$ is sufficiently small, yielding:
\begin{equation*}
    \begin{split}
        &\cos(\vy_{0,0},\vy_{t,\epsilon})= \\
        &(C_{\epsilon} + t E)(1-tE-\frac{t^2}{2}F + \frac{3t^2}{2}E^2 +O_t(t^3)) = \\
        & C_{\epsilon} + t(1-C_{\epsilon})E - \frac{t^2}{2}(C_{\epsilon}F-3 C_{\epsilon}E^2 + 2 E^2) + O_t(t^3).
    \end{split}
\end{equation*}
When $\epsilon =0$, we have $C_{\epsilon}=1$, then
$$
\cos(\vy_{0,0},\vy_{t,0}) = 1 -  \frac{t^2}{2}(F- E^2) + O(t^3).
$$
Then
The term $F-E^2$ can be further written as
$$
F-E^2 = F \sin^2\left( \mM\vx, g_{\vtheta} (\mM) \vx \right).
$$
When $\mM$ is invertible, the amplitude $F$ is a generalized Rayleigh quotient
$$
F = \frac{\vx^\intercal g_{\vtheta}(\mM)^2 \vx}{\vx^\intercal \mM^2 \vx},
$$
whose maximum is the generalized eigenvalue $\lambda_{\mathrm{max}}$ given by
$$
g_{\vtheta}(\mM)^2 \vx = \lambda_{\mathrm{max}} \mM^2 \vx,
$$
or $\mM^{-2} g_{\vtheta}(\mM)^2 \vx = \lambda_{\mathrm{max}}\vx$. When $\epsilon=0$, $\mM_{\epsilon}$ reduces to the noise-free graph filter $\mM$. For typical graph filters such as $\tilde{\mA}$ and $\tilde{\mL}$, $\lambda_{\mathrm{max}}$ is bounded from above by $\vert O_N(1)\vert$. Using
$\sin^2(\mM \vx, g_{\vtheta}(\mM))\leq 1$, 
we have 
$$
\cos(\vy_{0,0},\vy_{t,0}) \geq 1-t^2\vert O_N(1) \vert + O_t(t^3).
$$
and the claim follows.
\end{proof}

\begin{figure}[t]
\centering
\includegraphics[width=0.9\linewidth]{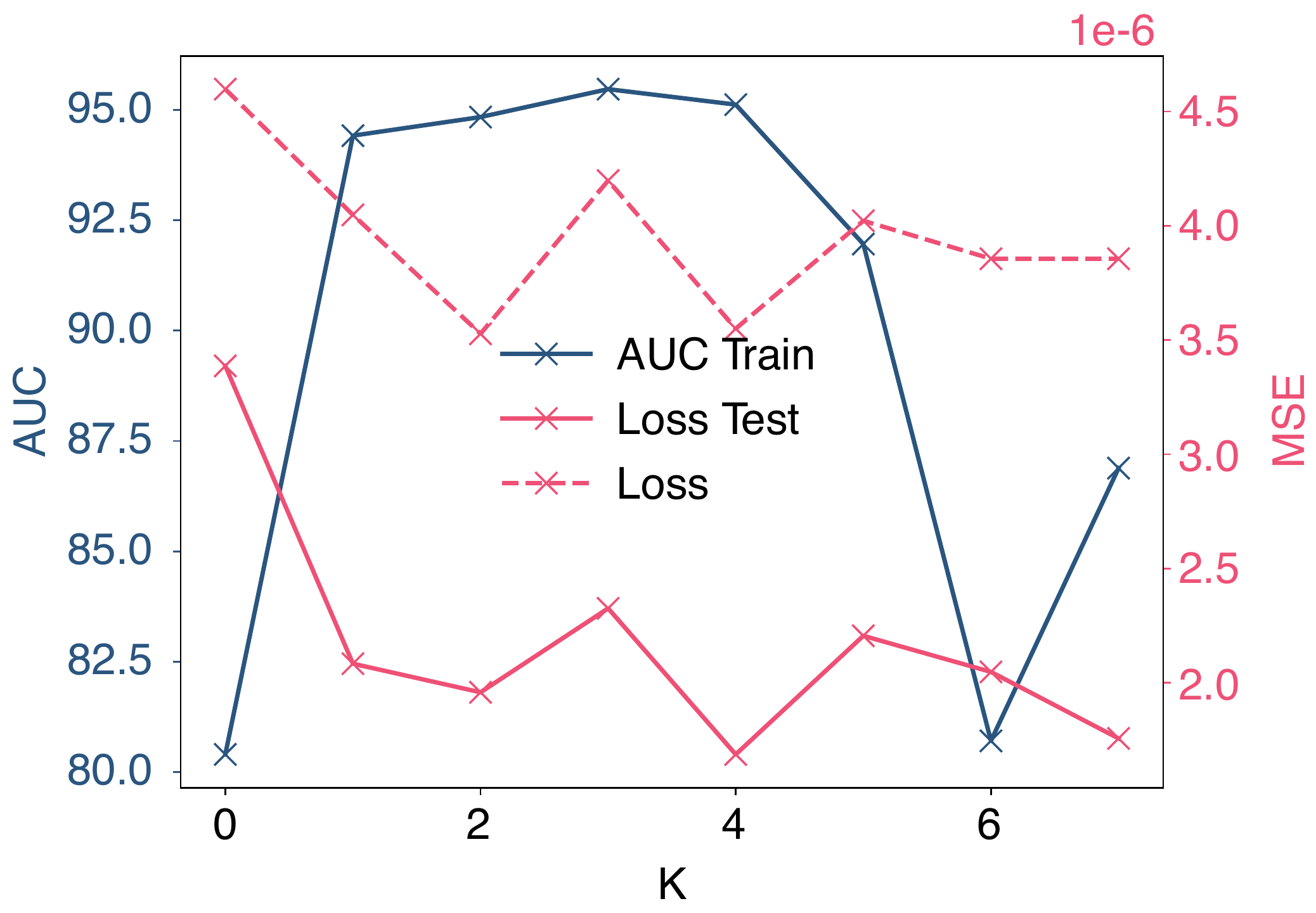}
    \caption{The dependency of model performance on polynomial order $K$. The results are averaged over $10$ independent runs.}
\label{fig:kdepend}
\end{figure}

\begin{table*}[t]
\centering
\caption{Relational inference accuracy measured by AUC. The last column shows results by stacking two message-passing layers.}
\begin{center}
\begin{tabular}{@{}lccccccc@{}}
\toprule
        Model & Graph & $\delta t$ & Volume & MI & TE & NRI & NRI-2 \\ 
 \midrule
\textbf{Spring} & ER-50 & $20$ & $15\times 10$ &  72.24&76.05    & \textbf{99.84±0.47} & 51.01±0.64 \\
\textbf{FJ} & ER-50 & $1$ & $20\times 10$  &53.66&83.64& \textbf{97.67±1.06} & 53.81±1.64\\
\bottomrule
\end{tabular}
\end{center}
\label{tab:stack}
\end{table*}

\begin{table}[t]
\centering
\begin{tabular}{@{}lccccccc@{}}
\toprule
Model & Graph & $T_o$ & Seed &  $\mA$ & $\tilde{\mA}$ & $g_{\theta}(\tilde{\mA})$\\ 
\midrule
\textbf{Diffusion} & ER-50 & $1$ & 0  & 85.90 & \textbf{92.04} & 90.47 \\
 &  & $1$ & 1    & 80.74 & 87.92 & \textbf{91.23}\\
 &  & $1$ & 2    & \textbf{93.86} & 84.49 & 80.77\\
 &  & $1$ & 3    & 87.79 & \textbf{93.04} & 90.86\\
 & ER-50 & $4$ & 0  & \textbf{91.51} & 86.89 & 80.21\\
 &  & $4$ & 1  & 86.71 & \textbf{87.07} & 82.27\\
  &  & $4$ & 2  & \textbf{78.35} & 70.96 & 66.95\\
 &  & $4$ & 3  & \textbf{85.50} & 76.51 &76.80\\
\bottomrule
\end{tabular}
\caption{Interaction graph AUC when using only the polynomial model.  The data volume is set to be $20\times 10$. Each row corresponds to an independent run with a different random seed. Boldface marks highest accuracy. \label{tab:seed}}
\end{table}

\section{Experiment Details}
\subsection{Dynamical Models and Data Generation}
\label{sec:data}

Here we summarize the details of dataset generation. For all experiments, we set the number of validation trajectories to $10$. The sampling interval and the number of sampled trajectories are identical to the training set. The number of sampled trajectories and time steps are shown in Table~1 in the main text. All trajectories are normalized to $[-1,1]$ before feed into training for NRI and GDP.
In Table~1, for ER graphs the link probability is set to $0.1$; for BA graphs, each new node is connected to $m=2$ existing nodes.\\[1mm]

\noindent{\textbf{Michaelis-Menten Kinetics}}
Michaelis–Menten kinetics~\cite{karlebach2008modelling} describes enzymatic reaction kinetics. The node state $x_i$ is one-dimensional and corresponds to concentrations of molecular species. Its evolution reads
$$
\dot{x}_i = - x_i + \frac{1}{\vert N_i \vert}\sum_{j\in N_i }\frac{x_j}{1+x_j},
$$
The equation is integrated with an ODE solver with the step size to $\delta t = 1$.\\[1mm]

\noindent{\textbf{Rössler Oscillators}}
A Rössler oscillator~\cite{rossler1976equation} is a three-dimensional dynamical system composed of one nonlinear and two linear equations. The model has been extended to graphs, where  node-wise Rössler oscillatosr are coupled through the edges. The model has been used to study synchronization in power grids and neuron oscillations. The state $\vx_i$ of each node is three-dimensional, where the evolution~\cite{casadiego2017model} reads:
\begin{equation*}
\begin{split}
 \dot{x}_{i,1} & = - x_{i,2} - x_{i,3} + \frac{1}{\vert N_i \vert} \sum_{j\in N_i} \sin \left(x_{j,1}\right),\\  
 \dot{x}_{i,2} &= x_{i,1} + 0.1 x_{i,2}, \\
 \dot{x}_{i,3} &= 0.1 + x_{i,3} \left(x_{i,3} - 18\right).
\end{split}    
\end{equation*}
When integrating the model, we set $\delta t =1$.\\[1mm]

\noindent \textbf{Diffusion}. The graph diffusion equation is $\dot{\vx} = \tilde{\mL} \vx$. We use the scale Laplacian $\mI - \tilde{\mL}$ and set $\delta t = 0.1$.\\[1mm]

\noindent \textbf{Springs}. 
The springs model describes particles connected by springs and interact via Hooke's law. The model was considered in NRI for testing the algorithm. We use the same setup as NRI, where particles are confined in a 2D box with elastic boundaries, and each particle is described by the location $\vr_i \in \R^2$ and velocity $\vv_i \in \R^2$.
The system's ODE reads 
\begin{equation}
    \frac{d\vr_i}{dt} = \vv_i,\qquad \frac{d\vv_i}{dt} = -k\sum_{j\in N_i} \left( \vr_i - \vr_j\right). 
\end{equation}
We use the trajectories of $\vr_i$ and $\vv_i$ as observed data. 
The size of the box is set to be $5$. We use the code for simulating the springs model provided by NRI\footnote{{https://github.com/ethanfetaya/NRI}}. The initial locations are sampled as i.i.d. Gaussians $\mathcal{N}(0,0.5)$, and the initial velocities is a random vector of norm $0.5$. Newton's equation of motion is integrated with a step size of $0.001$ and then subsamples each $100$ step to get the training and testing trajectories. Each sampled trajectory contains $49$ snapshots of the system states.\\[1mm]

\noindent \textbf{Kuramoto model}. 
Kuramoto model~\cite{kuramoto1975self} describes phase-coupled oscillators placed on a graph. The evoluation of the phase $\phi_i$ of vertex $i$ is described by the following ODE:
\begin{equation}
    \frac{d\phi_i}{dt} = \omega_i + k \sum_{j\in N_i} \sin (\phi_j - \phi_i),
\end{equation}
where $\omega_i$ are the intrinsic frequencies and $k$ is the coupling strength. We use the implementation and default parameters provided by~\citeauthor{lowe2022amortized}\footnote{https://github.com/loeweX/AmortizedCausalDiscovery}. The observables contains $d\phi_i/dt$, $\sin \phi_i$, $\phi_i$ and $\omega_i$. The sampling rate is $\delta t = 0.01$.\\[1mm]

\noindent{\textbf{Friedkin--Johnsen Dynamics}}. 
The Friedkin--Johnsen (FJ) dynamics~\cite{friedkin1990social} is a classical model for describing opinion formation~\cite{abebe2018opinion,okawa2022predicting}, polarization and filter bubble~\cite{chitra2020analyzing} in social networks. In FJ model, each vertex at time $t$ holds an ``expressed'' opinion $x^t_i\in[-1,1]$ and an internal opinion $s_i\in [-1,1]$. While $s_i$ does not change over time, $x_i$ evolves according to the rule
\begin{equation}
    x^{t+1}_i = \frac{s_i + \sum_{j\in N_i} x^{t}_j}{1 + \vert N_i\vert},
\end{equation}
where $\vert N_i\vert$ is the degree of vertex $i$. The model reaches an equilibrium state in the long-time limit where all vertices hold a constant opinion.
We study relational inference for the model in its transient state.
We sample the initial expressed and internal opinions from a uniform distribution in $[-1,1]$ and update vertex states for some steps to generate the dataset. We take both $s_i$ and $x_i^t$ as observed quantities.\\[1mm]

\noindent{\textbf{Coupled Map Networks}}. Coupled Map Networks~\cite{garcia2002coupled} is a discrete-time system that can generate chaotic dynamics,
$$
x_i^{t+1} = (1-\epsilon) f(x_i^{t}) + \frac{\epsilon}{\vert N_i \vert } \sum_{j\in N_i} f(x_j^t),\ f(x_i^t) = \eta x_i^t (1-x_i^t). 
$$
We set $\epsilon = 0.2$ and $\eta = 3.5$. We use the implementation provided by~\citeauthor{zhang2022universal}\footnote{https://github.com/kby24/AIDD}.\\[1mm]

\noindent{\textbf{Netsim}}
Netsim~\cite{smith2011network} simulates blood-oxygen-level-dependent imaging data across
different regions within the human brain. The interaction graph describes the directed connectivity between brain regions. The dataset volume is $50\times 200$, and the graph has $n=15$ nodes. As in \citeauthor{lowe2022amortized}, we do not split the data into training/validation sets but use all the data in each phase. 

\begin{figure}[t]
\centering
\includegraphics[width=0.8\linewidth]{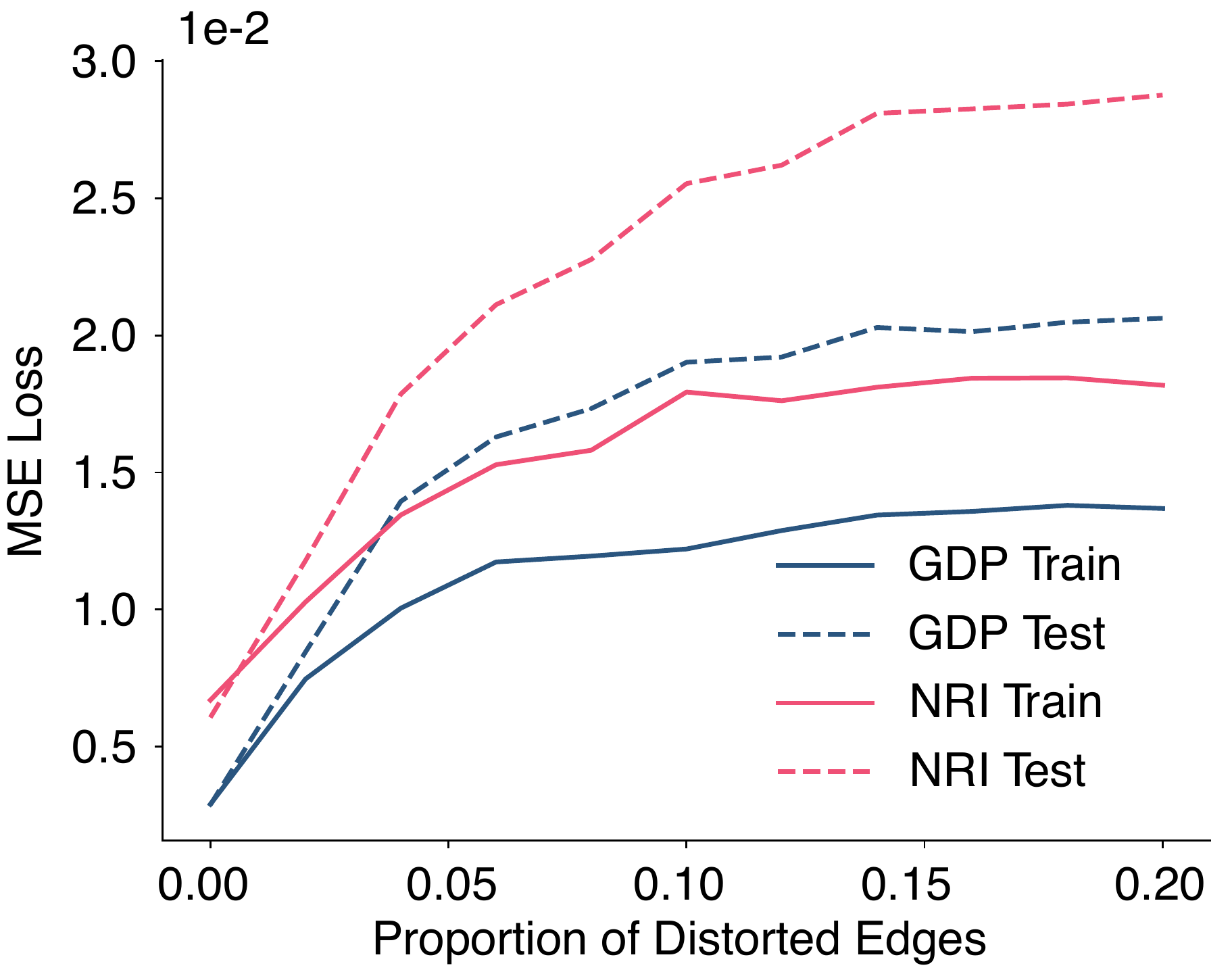}
    \caption{The training and test loss of neural surrogate with fixed graph. The graph is an ER graph with $n=20$ and $p=0.5$, and the sampling interval is set to be $\delta t=1$.
    The volume of training data is $50\times 10$. We randomly flip a fraction of edges, and train the model on the fixed distorted graph.}
\label{fig:distortion}
\end{figure}

\begin{figure}[t]
\centering
\includegraphics[width=0.8\linewidth]{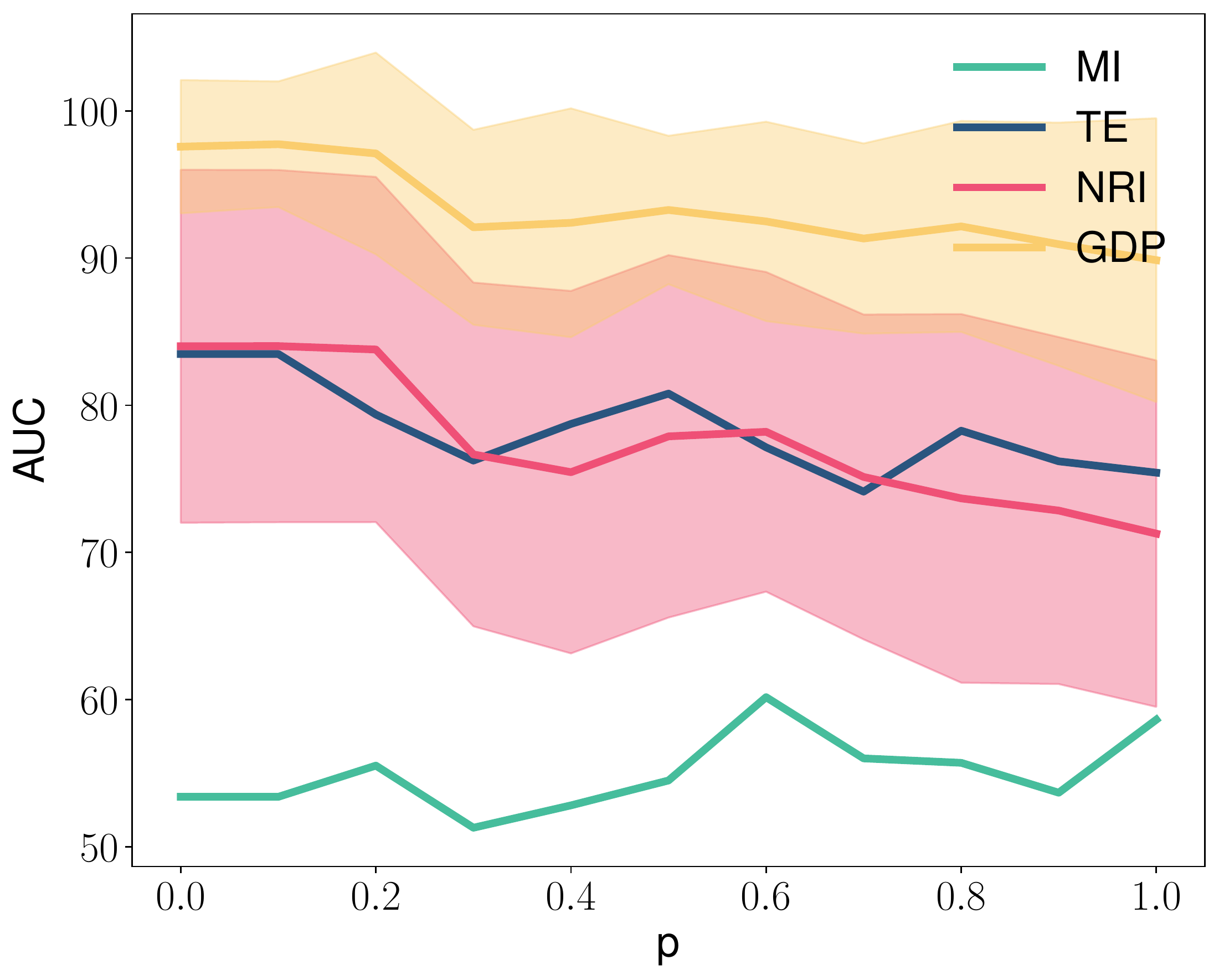}
    \caption{Influence of graph properties on inference accuracy for the diffusion model. The experiment is performed on Watts-Storgatz graphs with $n=30$, $k=2$, and we vary the rewiring probability $p$. The shaded region denotes the standard deviation.}
\label{fig:feature}
\end{figure}

\subsection{Baseline Methods and Implementation Details}
\label{sec:baseline}

\noindent \textbf{NRI}. NRI was originally designed in the amortized setting when each trajectory has a distinct interaction graph. It was extended to the non-amortized case for relational~\cite{zhang2022universal} and causal~\cite{lowe2022amortized} inference. In the non-amortized setting, the graph encoder is not necessary. As we do not require make causal claims, we do not use the test-time adaption technique in~\citeauthor{lowe2022amortized}, but infer the interaction graph during training time. The predicted interaction graph is picked through the MSE loss in a validation set. We re-implement NRI based on ~\citeauthor{lowe2022amortized} for the non-armotized case with binary edge types, and use their default hyperparameters. More details can be found in our implementation available online.\\[1mm]

\noindent \textbf{Mutual Information (MI)}. Mutual information of two random variables is defined as 
\begin{equation}
    MI_{i,j} = \sum_{x_i,x_j} P (x_i,x_j) \log \frac{P (x_i,x_j)}{P(x_i)P(x_j)},
\end{equation}
where $P (x_i,x_j)$ is the joint distributions of $x_i$ and $x_j$ sampled from the system trajectories, and $P(x_i)$ ($P(x_j)$) are their marginals. The random variables are defined on each vertex for the relational inference problem. We compute the mutual information for all pairs of nodes, which is further taken as the score for the presence of an edge. The node state distributions are approximated from the sampled trajectories.\\[1mm]

\noindent \textbf{Transfer Entropy (TE)}
The transfer entropy is defined as 
$$
TE_{i,j} = H(x_j^t\vert x_j^{t-1}) - H(x_j^t\vert x_j^{t-1}, x_i^{t-1}),
$$
where $H(\cdot \vert \cdot)$ is the conditional entropy. In order to estimate the conditional entropy numerically, the observed trajectories are discretized by data binning. The number of bins is picked from $\{2,200\}$, and we report the higher accuracy.

Unlike NRI and GDP, MI and TE do not require a validation set. For a fair comparison, we compute MI and TE on the training and validation set combined. When there are multiple variables, we compute MI or TE for each variable and take the average as the prediction. 
We use the implementations from netrd package\footnote{https://netrd.readthedocs.io/en/latest/} for MI and TE.

\begin{table*}[t]
\centering
\caption{Relational inference accuracy measured by AUC on directed graphs for different methods.}
\begin{center}
\begin{tabular}{@{}lccccccc@{}}
\toprule
        Model & Graph & $\delta t$ & Volume & MI & TE & NRI & GDP \\ 
 \midrule
        \textbf{Diffusion} & ER(D)-50 & $1$ & $ 20\times 10$ & 53.89 & 52.73 & 68.67±12.75 & \textbf{90.23±2.98} \\ 
        \textbf{Kuramoto} & ER(D)-50 & $1$ & $20 \times 30$ & 51.32 & 52.72 & 65.92±9.04 & \textbf{72.60±8.49} \\ 
\bottomrule
\end{tabular}
\end{center}
\label{tab:directed}
\end{table*}

\begin{table*}[t]
\centering
\caption{Relational inference accuracy measured by AUC on real-world data.}
\begin{center}
\begin{tabular}{@{}lccccccc@{}}
\toprule
        Model & Graph & $\delta t$ & Volume & MI & TE & NRI & GDP \\ 
 \midrule
        \textbf{Traffic} & METR-LA-207 & $1$ & $ 23974\times 12$ & 58.21 & 73.31 & 70.54 & \textbf{76.79} \\ 
        \textbf{MM} & Gene-100 & $1$ & $250 \times 10$ & 75.23 & 53.82 & 63.13 & \textbf{78.57} \\ 
\bottomrule
\end{tabular}
\end{center}
\label{tab:real}
\end{table*}

\subsection{Hyperparameters for GDP}
\label{sec:hyper}
All the results of GDP reported in the main text  (except the Netsim dataset) are obtained under identical neural network architecture and hyperparameters. We use Adam optimizer to train the model. The learning rate is set to be $0.1$ for the graph generator and $0.0005$ for the dynamics surrogate. The inverse temperature parameter $\beta$ of the graph generator in Equation~2 of the main text is set to be $0.5$. For all other datasets except Netsim, the polynomial filter is truncated at $K=4$, and the number of GNN layers is set to $1$.

The Netsim dataset is stochastic, and we introduce some addition techniques of temporal smoothing. In particular, we set $k=6$ with $4$ GNN layers. In order to make the outputs in the first three hidden layers stable, we perform linear interpolation on consecutive data points and fit the hidden layers outputs to the interpolated data. In addition, we replace the data at each time step with the average on its neighbouring time steps. More details can be found in our implementation available online. We leave a more detailed analysis of stochastic dynamics for future work.

\subsection{Graph / Complement Graph Ambiguity} \label{sec:ambiguity}
The neural surrogate essentially classify the edges into two types, but cannot determine which type correspond to the positive edges and which to the negative. In fact Equation~3 in the main text is permutation invariant to the edge-types $a\in\{0,1\}$. Therefore, there is the graph/complement graph ambiguity problem. Moreover, the effective interaction graph itself can be negatively correlated to the adjacency matrix. For example, for the linear ODE $\dot{\vx} = \beta \tilde{\mA} \vx$, a negative $\beta$ makes the effective interaction negatively correlated with the interaction graph, as shown in Figure~2 in the main text. In practical applications,  we can use sparsity prior, or compare with the statistical methods to decide which one as the prediction. In Table~1 in the main text, we simply report $\max\{\text{AUC}, \text{1-AUC}\}$.

\section{More Results}

\subsection{Insufficiency of stacking multiple message-passing layers}
We consider the approach by naively stacking two message-passing layers for the undersampling case. The results are shown in Table~\ref{tab:stack} for Spring and FJ dynamics. By stacking two message-passing layers, the AUC is close to 50\%, indicating that the approach degrades completely.

\subsection{Dependency on Polynomial Orders}\label{sec:kdepend}
We take the Michaelis-Menten Kinetics as an example to test the dependency on the polynomial order $K$. We consider an ER graph with $n=20$ nodes and generate the 
trajectories at sampling interval $\delta t =4$. The volume of training/validation/test data are $50\times 10$/$10\times 10$/$10\times 10$, respectively. The training and test errors and graph inference AUC versus $K$ are shown in Figure~\ref{fig:kdepend}. The model performance is optimized at an intermediate value of $K$, which shows that including higher-order terms can help in recovering the interactions.

\subsection{Ablation Study of Using Only Polynomial filters}\label{sec:ablation}
Our proposed model reduces to NRI if the polynomial filter part is removed. Therefore, the ablation study mainly focuses on the case when
we only use a polynomial filter. As we have discussed in the main text, there are, in general, multiple graph matrices that can result in the same polynomial filters, so it is hard to control which one will converge. We use only the  $g_{\vtheta}$ model to verify the phenomenon and conduct experiments with different random seeds in Pytorch. We use the diffusion model as an example and record the predictions generated by $\mA$, $\tilde{\mA}$ and $g_{\theta}(\tilde{\mA})$ in Table~\ref{tab:seed}. Any of these three matrices has predicted the interaction graph best under one random seed, suggesting that the convergence is hard to control if we only use the polynomial filter. As we only list three related graph matrices, other better matrices might not be included. Therefore, an $\mA$ model is necessary to stabilize the solutions.

\subsection{Local Attractiveness of Ground Truth Graph}\label{sec:attractor}
The experimental results show that the ``good'' graph for dynamics prediction is also correct. In other words, the true graph is an attractor (at least locally) when we make the graph trainable. We test the local attractiveness of the ground truth graph via the following experiment. In particular, we train the neural surrogate by keeping the graph fixed on each set of graphs. The set contains a ground truth graph as well as many distorted graphs. We randomly select a fraction of the edges for distortion and flip its edge type. We consider the Diffusion model and generate the training trajectories by simulating with the ground truth graph and train the neural surrogate on each graph for $10$ times independently. The average train and test MSE error versus the proportion of distorted edges is shown in Figure~\ref{fig:distortion}. From Figure~\ref{fig:distortion}, for both NRI and GDP, the train and test error in general increase with the proportion of distorted edges, which reflects that with the ground truth graph, we can better predict the dynamics.

\subsection{Results on directed graphs}\label{sec:directed}
In the main text, only the Netsim dataset is directed. We perform additional experiments to verify the performance of GDP in this case. In particular we take Diffusion and Kuramoto dynamics on directed ER graphs as examples. The results are shown in Table~\ref{tab:directed}. GDP outperforms other baselines on the considered examples.

\subsection{Results on real-world data}\label{sec:real}

We carry out additional experiments on real-world data. We consider two new datasets: (\romannumeral1) A dataset of traffic information collected from loop detectors on the Los Angeles County~\cite{li2018diffusion}. We selected 207 sensors and 4 months of data from Mar 1st 2012 to Jun 30th 2012 for the experiment. (\romannumeral2) A gene regulatory network of S. cerevisiae yeast and Michaelis-Menten model on the regulation data~\cite{schaffter2011genenetweaver}. 

Both new sets of results are shown in Table~\ref{tab:real}. From the results, GDP also performs well on these new datasets, outperforming all baselines. More experiments on real data requires considerable domain-specific knowledge (and maybe according model tweaks), and we leave it to future works.

\subsection{Impacts of Graph Structure on Inference Accuracy}

Previous works on graph dynamical systems show that the dynamical behaviour relies strongly on the graph structure. For example, the Kuramoto model has an incoherent-coherent phase transition, where the transition threshold depends on the graph structure intricately. In the coherent phase, we expect the inference to be more challenging as all nodes are synchronized and have identical states. Therefore, the influence of graph properties should be dynamics-dependent. One way to test this influence experimentally is by using graphs with tunable graph properties. We conduct experiments with the Diffusion model on the Watts-Strogatz graph. The graph starts with a k-nearest neighbour ring, and each edge is randomly rewired with probability $p$. When $p$ grows, the graph becomes more random. The inference accuracy versus $p$ is shown in Figure~\ref{fig:feature}. We can observe that the inference accuracy decreases as the graph becomes increasingly random. 

\end{document}